\documentclass[conference]{IEEEtran}
\IEEEoverridecommandlockouts
\usepackage{cite}
\usepackage{amsmath,amssymb,amsfonts}
\usepackage{algorithmic}
\usepackage{graphicx}
\usepackage{textcomp}
\usepackage{xcolor}

\usepackage{subfigure}
\usepackage{bbm}
\usepackage{mathrsfs}
\usepackage{indentfirst}
\usepackage{multirow}

\usepackage{amsthm}
\usepackage{amsmath}

\newtheorem{theorem}{Theorem}
\newtheorem{definition}{Definition}
\newtheorem{lemma}{Lemma}

\theoremstyle{definition}

\def\Real{\mathbb{R}}

\def\BibTeX{{\rm B\kern-.05em{\sc i\kern-.025em b}\kern-.08em
    T\kern-.1667em\lower.7ex\hbox{E}\kern-.125emX}}
\begin{document}

\title{Policy Gradients for Probabilistic Constrained Reinforcement Learning
\thanks{This work was supported by the IBM Rensselaer Research Collaboration.}
}

\author{\IEEEauthorblockN{Weiqin Chen}
\IEEEauthorblockA{\textit{Department of Electrical,} \\
\textit{Computer, and Systems Engineering} \\
\textit{Rensselaer Polytechnic Institute}\\
Troy, NY, USA \\
chenw18@rpi.edu}
\and
\IEEEauthorblockN{Dharmashankar Subramanian}
\IEEEauthorblockA{\textit{IBM T. J. Watson Research Center}\\
Yorktown Heights, NY, USA \\
dharmash@us.ibm.com}
\and
\IEEEauthorblockN{Santiago Paternain}
\IEEEauthorblockA{\textit{Department of Electrical,} \\
\textit{Computer, and Systems Engineering} \\
\textit{Rensselaer Polytechnic Institute}\\
Troy, NY, USA \\
paters@rpi.edu}
}

\maketitle

\begin{abstract}
This paper considers the problem of learning safe policies in the context of reinforcement learning (RL). In particular, we consider the notion of probabilistic safety. This is, we aim to design policies that maintain the state of the system in a safe set with high probability. This notion differs from cumulative constraints often considered in the literature. The challenge of working with probabilistic  safety is the lack of expressions for their gradients. Indeed, policy optimization algorithms rely on gradients of the objective function and the constraints. To the best of our knowledge, this work is the first one providing such explicit gradient expressions for probabilistic constraints. It is worth noting that the gradient of this family of constraints can be applied to various policy-based algorithms. We demonstrate empirically that it is possible to handle probabilistic constraints in a continuous navigation problem.
\end{abstract}

\begin{IEEEkeywords}
reinforcement learning, probabilistic constraint, safe policy,  policy gradient
\end{IEEEkeywords}

\section{Introduction}
\label{sec_intro}
Reinforcement learning (RL) has gained traction as a solution to the problem of computing policies to perform challenging and high-dimensional tasks, e.g., playing video games~\cite{mnih2013playing}, mastering Go~\cite{silver2017mastering},  robotic manipulation~\cite{levine2016end} and locomotion~\cite{duan2016benchmarking}, etc. However, in general, RL algorithms are only concerned with maximizing a cumulative reward~\cite{watkins1992q,sutton1999policy}, which may lead to risky behaviors~\cite{garcia2015comprehensive} in realistic domains such as robot navigation~\cite{kahn2018self}.

Taking into account the safety requirements motivates the development of policy optimization under safety guarantees~\cite{geibel2006reinforcement,kadota2006discounted,chow2017risk}. Some approaches consider risk-aware objectives or regularized solutions where the reward is modified to take into account the safety requirements~\cite{howard1972risk,sato2001td,geibel2005risk}. 
Other formulations include Constrained Markov Decision Processes (CMDPs)~\cite{altman1999constrained}
where additional cumulative (or average) rewards need to be kept above a desired threshold. This approach has been commonly used to induce safe behaviors~\cite{borkar2005actor,bhatnagar2012online,liang2018accelerated, achiam2017constrained, tessler2018reward, yang2020projection, zhang2020first, liu2020ipo, zhang2022penalized}. To solve these constrained problems, primal-dual algorithms~\cite{borkar2005actor,bhatnagar2012online,liang2018accelerated,achiam2017constrained, tessler2018reward} combined with classical and state-of-the-art policy-based algorithms, e.g., REINFORCE~\cite{williams1992simple},  DDPG~\cite{lillicrap2015continuous}, TRPO~\cite{schulman2015trust}, PPO~\cite{schulman2017proximal} are generally used.


In this cumulative constraint setting, safety violations are acceptable as long as the amount of violations does not exceed the desired thresholds. This makes them often not suitable for safety-critical applications. For instance, in the context of autonomous driving, even one collision is unacceptable. 

A more suitable notion of safety in this context is to guarantee that the whole trajectory of the system remains within a set that is deemed to be safe (see e.g.,~\cite{castellano2022reinforcement}). Ideally, one would like to achieve this goal for every possible trajectory. This being an ambitious goal, in this work we settle for solutions that guarantee every time safety with high probability. We describe this setting in detail in Section \ref{sec_problem_formulation}. 
 
The main challenge in solving problems under probabilistic safety constraints is that explicit policy gradient-like expressions for such constraints are not readily available. Indeed in \cite{tessler2018reward,paternain2022safe} replace this constraint with a suitable lower bound.
In \cite{peng2021separated} an approximation of the gradient is also provided. These limitations, prevent us  from running the aforementioned policy-based algorithms. In Section~\ref{Gradient of the Probabilistic Constraint}, we present the main contribution of this work: an expression for the gradient that enables stochastic approximations. Other than concluding remarks, the paper finishes (Section~\ref{Numerical_Experiments}) with numerical examples that demonstrate the use of the probabilistic safe gradient to train safe policies in a navigation task.

\section{Problem Formulation}
\label{sec_problem_formulation}
In this work, we consider the problem of finding optimal policies for Markov Decision Processes (MDPs) under probabilistic safety guarantees. In particular, we are interested in situations where the state transition distributions are unknown and thus the policies need to be computed from data. An MDP~\cite{sutton2018reinforcement} is defined by a tuple ($\mathcal{S}, \mathcal{A}, r, \mathbb{P}, \mu, T$), where $\mathcal{S}$ is the state space, $\mathcal{A}$ is the action space, $r: \mathcal{S} \times \mathcal{A} \to \Real$ is the reward function describing the quality of the decision, {$\mathbb{P}_{s_t \to s_{t+1}}^{a_t} (\hat{\mathcal{S}}):=\mathbb{P}(s_{t+1} \in \hat{\mathcal{S}} \mid s_t, a_t)$ where $\hat{\mathcal{S}} \subset \mathcal{S}, s_t \in \mathcal{S}, a_t \in \mathcal{A}, t\in \mathbb{N}$ is the transition probability describing the dynamics of the system, $\mu (\hat{\mathcal{S}}):= \mathbb{P}(s \in \hat{\mathcal{S})}$ is the initial state distribution, }and $T$ is the time horizon. The state and action at time $t$ are random variables denoted respectively by $S_t$ and $A_t$. A \emph{policy} is a conditional distribution $\pi_\theta (a|s)$ parameterized by $\theta \in \Real^d$ (for instance the weights and biases of neural networks), from which the agent draws action $a \in \mathcal{A}$ when in the corresponding state $s \in \mathcal{S}$. In the context of MDPs the objective is to find a policy that maximizes the value function. The latter is defined as 
\begin{equation}\label{eqn_val_func_finite}
    V (\theta) =  \mathbb{E}_{\mathbf{a} \sim \pi_\theta (\mathbf{a}|\mathbf{s}), S_0 \sim \mu} \left[\sum\limits_{t=0}^{T} r(S_t, A_t)\right],
\end{equation}
where $\mathbf{a}$ and $\mathbf{s}$ denote the sequences of actions and states for the whole episode, this is, from time $t=0$ to $t=T$. The subscripts of this expectation are omitted in the remaining of the paper for simplicity.


In this paper, we are concerned with imposing safety requirements. In particular, we focus on the notion of probabilistic safety which we formally define next.
%
%
\begin{definition}
\label{definition_safety}
A policy $\pi_\theta$ is $(1-\delta)$-safe for the set $\mathcal{S}_\text{safe} \subset \mathcal{S}$ if and only if $\mathbb{P} \left(\bigcap\limits_{t=0}^{T} \{ S_t \in \mathcal{S}_\text{safe}\} |\pi_\theta \right) \geq 1-\delta$.
\end{definition}
%
%
With this definition, we formulate the following probabilistic safe RL problem as a constrained optimization problem
\begin{align}\label{eqn_problem1}
 P^\star = &\max\limits_{\theta \in \Real^d} \, V(\theta)      \nonumber \\
 &\text{s.t.} \quad \mathbb{P} \left(\bigcap\limits_{t=0}^{T} \{ S_t \in \mathcal{S}_\text{safe}\} |\pi_\theta \right) \geq 1-\delta. 
\end{align}

{Note that the previous problem differs from most of the safe RL literature that work with the cumulative constraint setting (see e.g., ~\cite{achiam2017constrained, yang2020projection, zhang2022penalized, tessler2018reward}). To solve problem~\eqref{eqn_problem1}, it is conceivable to employ {gradient-based} methods e.g., regularization~\cite{censor1977pareto} or primal-dual~\cite{arrow1958studies} to achieve local optimal solutions. For instance, consider the regularization method with a fixed penalty. This is, for $\lambda>0$ we formulate the following \emph{unconstrained} problem as an approximation to the \emph{constrained} problem \eqref{eqn_problem1}
\begin{align}\label{eqn_agmt_obj}
 \mathbb{E} \left[\sum\limits_{t=0}^{T} r(S_t, A_t) \right] \! + \! \lambda \! \left(\!  \mathbb{P}  \left(\bigcap\limits_{t=0}^{T} \{ S_t \in \mathcal{S}_\text{safe}\} |\pi_\theta \right) \!- \! (1-\delta) \! \right).
\end{align}
%
Note that $\lambda$ trades-off safety and task performance. Indeed, for large values of $\lambda$ solutions to \eqref{eqn_agmt_obj} will prioritize safe behaviors, whereas for small values of $\lambda$ the solutions will focus on maximizing the { expected cumulative rewards}.

Then, to solve problem~\eqref{eqn_agmt_obj} locally, gradient ascent~\cite{bertsekas1997nonlinear} or its stochastic versions are generally employed. Note that the gradient of the first term in~\eqref{eqn_agmt_obj} can be computed using the Policy Gradient Theorem~\cite{sutton1999policy}. Nevertheless, the lack of an expression for the gradient of the probabilistic safety constraint, i.e., $\nabla_\theta \mathbb{P} \left(\bigcap\limits_{t=0}^{T} \{ S_t \in \mathcal{S}_\text{safe}\} |\pi_\theta \right)$ prevents us from applying {the gradient ascent} family of methods to solve \eqref{eqn_agmt_obj}. In the next section, we provide such an expression for the gradient that allows us to overcome this limitation.



\section{The Gradient of Probabilistic Constraints}
\label{Gradient of the Probabilistic Constraint}
%
We start this section by defining an important quantity in what follows next. For any $t$ such that $0\leq t\leq T$ define 
\begin{equation}\label{def_G_cumulative_product}
G_t = \prod_{u=t}^T\mathbbm{1}\left(S_u\in\mathcal{S}_{\text{safe}}\right).    
\end{equation}
Having defined this quantity we are now in conditions of providing an expression for the gradient of the probabilistic constraint. This is the subject of the following Theorem.
\begin{theorem}\label{theorem_safe_policy_gradient}
Let $S_0 \in \mathcal{S}_\text{safe}$, the gradient of the probability of being safe for a given policy $\pi_\theta$ yields 
\begin{align}\label{eqn_theorem}
   &\nabla_\theta \mathbb{P} \left(\bigcap\limits_{t=0}^{T} \{ S_t \in \mathcal{S}_\text{safe}\} |\pi_\theta, S_0 \right) \nonumber \\
   &=\mathbb{E}\left[ \sum\limits_{t=0}^{T-1}G_1\nabla_{\theta}\log\pi_\theta(A_t\mid S_t)\mid \pi_\theta, S_0\right].
\end{align}
\end{theorem}
\begin{proof}

We proceed by presenting and proving the following two technical lemmas (Lemma~\ref{lemma_safe_policy_gradient_G1} and Lemma~\ref{lemma_nabla_E_G1_S0_GT_ST-1}).

\begin{lemma}
\label{lemma_safe_policy_gradient_G1}
Given $S_{t-1} \in \mathcal{S}_\text{safe}$ and $G_{t}, t=1,2,\cdots, T-1$ defined in \eqref{def_G_cumulative_product}, it holds that
\begin{align}\label{eqn_recursive_gradient}
\nabla_\theta\mathbb{E}\left[G_{t}\mid S_{t-1}\right] &= \mathbb{E}\left[\nabla_\theta\mathbb{E}\left[G_{t+1}\! \mid \! S_{t}\right]\mathbbm{1}\left(S_{t}\in\mathcal{S}_{\text{safe}}\right) \! \mid \! S_{t-1} \right] \nonumber \\
&+ \mathbb{E}\left[G_{t}\nabla_{\theta}\log\pi_\theta(A_{t-1}\mid S_{t-1}) \mid S_{t-1}\right].
\end{align}
\end{lemma}
\begin{proof}
We start the proof by rewriting the expectation of $G_1$ using the towering property of the expectation
\begin{align}
    \mathbb{E}\left[G_1\mid S_0\right] &=\mathbb{E}\left[\mathbb{E}\left[G_1\mid S_1\right]\mid S_0\right] \nonumber \\
    &= \mathbb{E}\left[\mathbb{E}\left[G_2\mathbbm{1}\left(S_1 \in \mathcal{S}_{\text{safe}}\right)\mid S_1\right]\mid S_0\right], 
\end{align}
where the second equality follows from \eqref{def_G_cumulative_product}. Since $S_1$ is measurable with respect to the $\sigma$-algebra $\mathcal{F}_{1}$ it follows that
\begin{equation}
    \mathbb{E}\left[G_1\mid S_0\right] = \mathbb{E}\left[\mathbb{E}\left[G_2\mid S_1\right] \mathbbm{1}\left(S_1 \in \mathcal{S}_{\text{safe}}\right)\mid S_0\right].
\end{equation}
Rewriting the outer expectation in terms of the probability distribution of $S_1$, the previous expression reduces to
\begin{align}\label{eqn_conditional_return}
    \mathbb{E}\left[G_1\mid S_0\right] &= \int_{\mathcal{S}}\mathbb{E}\left[G_2\mid s_1\right]\mathbbm{1}\left(s_1 \in \mathcal{S}_{\text{safe}}\right) p\left(s_1 \mid S_0\right) \, ds_1 \nonumber \\
    &=\int_{\mathcal{S}\times \mathcal{A}} \mathbb{E}\left[G_2\mid s_1\right]\mathbbm{1}\left(s_1 \in \mathcal{S}_{\text{safe}}\right) \nonumber  \\
    &\quad \quad p(s_1 \mid S_0, \, a_0) \pi_{\theta}(a_0 \mid S_0) \, ds_1 da_0,
\end{align}
where the last equality follows from $p(s_1\mid S_0) = \int_{\mathcal{A}} p(s_1\mid S_0,\,a_0) \pi_{\theta}(a_0\mid S_0) \, da_0.$
%
%
%
Taking the gradient of the previous expression with respect to the policy parameters $\theta$ results in
\begin{align}\label{eqn_gradient_1}
    \nabla_\theta \mathbb{E}\left[G_1\mid S_0\right]  &= \int_{\mathcal{S}\times \mathcal{A}}\nabla_{\theta}\left(\mathbb{E}\left[G_2\mid s_1\right]\right)\mathbbm{1}\left(s_1 \in \mathcal{S}_{\text{safe}}\right) \nonumber \\
    & \quad\quad\quad\quad p(s_1 \mid S_0, \, a_0) \pi_{\theta}(a_0\mid S_0) \, ds_1 da_0  \nonumber \\
    &+\int_{\mathcal{S}\times \mathcal{A}}\mathbb{E}\left[G_2\mid s_1\right]\mathbbm{1}\left(s_1 \in \mathcal{S}_{\text{safe}}\right)  \\
    & \quad\quad\quad\quad p(s_1 \! \mid \! S_0, \, a_0) \nabla_{\theta}\pi_{\theta}(a_0 \! \mid \! S_0) \, ds_1 da_0 \nonumber.
\end{align}
Notice that the first in the right-hand side of the previous expression can be presented by
\begin{align}\label{eqn_pre_recursion_1}
&\int_{\mathcal{S}\times \mathcal{A}} \nabla_{\theta}\left(\mathbb{E}\left[G_2\mid s_1\right]\right)\mathbbm{1}\left(s_1 \in \mathcal{S}_{\text{safe}}\right) \nonumber \\
&\quad \quad \quad p(s_1 \mid S_0, \, a_0) \pi_{\theta}(a_0\mid S_0) \, ds_1 da_0  \nonumber \\
&= \mathbb{E}\left[\nabla_\theta\mathbb{E}\left[G_2\mid S_1\right]\mathbbm{1}\left(S_1\in\mathcal{S}_{\text{safe}}\right) \mid S_0 \right].
\end{align}

The second term using the ``log-trick'' yields
%
\begin{align}\label{eqn_logtrick}
    &\int_{\mathcal{S}\times \mathcal{A}}\mathbb{E}\left[G_2\mid s_1\right]\mathbbm{1}\left(s_1 \in \mathcal{S}_{\text{safe}}\right) \nonumber \\
    &\quad \quad \quad p(s_1 \mid S_0, \, a_0) \nabla_{\theta}\pi_{\theta}(a_0\mid S_0) \, ds_1 da_0 \nonumber \\
    &=\int_{\mathcal{S}\times \mathcal{A}}\mathbb{E}\left[G_2\mid s_1\right]\mathbbm{1}\left(s_1 \in \mathcal{S}_{\text{safe}}\right) p(s_1 \mid S_0, \, a_0) \nonumber \\
    &\quad \quad \quad \quad \pi_{\theta}(a_0\mid S_0) \nabla_\theta \log\pi_\theta(a_0\mid S_0) \, ds_1 da_0.
\end{align}
{
Likewise, since $s_1$ is measurable with respect to the $\sigma$-algebra $\mathcal{F}_{1}$, \eqref{eqn_logtrick} can be simplified as $\int_{\mathcal{S}\times \mathcal{A}}\mathbb{E}\left[G_1 \mid s_1\right] p(s_1 \mid S_0, \, a_0) \pi_{\theta}(a_0\mid S_0) \nabla_\theta \log\pi_\theta(a_0\mid S_0) \, ds_1 da_0,$
which is also equivalent to $ \mathbb{E}\left[\mathbb{E}\left[G_1 \mid S_1\right] \nabla_\theta \log\pi_\theta(A_0\mid S_0) \mid S_0 \right]$.
}
%
Since $\log\pi_\theta(A_0\mid S_0)$ is measurable given $S_1$, we have
\begin{align}
    &\mathbb{E}\left[\mathbb{E}\left[G_1 \mid S_1\right] \nabla_\theta \log\pi_\theta(A_0\mid S_0) \mid S_0 \right] \nonumber \\
    &= \mathbb{E}\left[\mathbb{E}\left[G_1 \nabla_\theta \log\pi_\theta(A_0\mid S_0) \mid S_1\right]  \mid S_0 \right].
\end{align}
Using the towering property of the expectation the previous expressions yield
\begin{align}\label{eqn_pre_recursion_2}
   &\int_{\mathcal{S}\times \mathcal{A}}\mathbb{E}\left[G_2\mid s_1\right]\mathbbm{1}\left(s_1 \in \mathcal{S}_{\text{safe}}\right) p(s_1 \mid S_0, \, a_0) \nonumber \\
   &\quad \quad \quad \nabla_{\theta}\pi_{\theta}(a_0\mid S_0) \, ds_1 da_0 \nonumber \\
   &= \mathbb{E}\left[G_1 \nabla_\theta \log\pi_\theta(A_0\mid S_0) \mid S_0 \right]. 
\end{align}
Then, combining \eqref{eqn_pre_recursion_1} with \eqref{eqn_pre_recursion_2} yields
\begin{align}\label{eqn_appendix_nabla_E_G1_S0}
\nabla_\theta\mathbb{E}\left[G_1\mid S_0\right] &= \mathbb{E}\left[\nabla_\theta\mathbb{E}\left[G_2\mid S_1\right]\mathbbm{1}\left(S_1\in\mathcal{S}_{\text{safe}}\right) \mid S_0 \right] \nonumber \\
&+ \mathbb{E}\left[G_1\nabla_{\theta}\log\pi_\theta(A_0\mid S_0)\mid S_0\right].
\end{align}
Repeating the process above $i$ times for $1 \leq i \leq T-1$, we obtain the following recursive definition of the gradient of the probability in \eqref{eqn_problem1} with respect to $\theta$
\begin{align}
\nabla_\theta\mathbb{E}\left[G_{i}\mid S_{i-1}\right] &= \mathbb{E}\left[\nabla_\theta\mathbb{E}\left[G_{i+1}\! \mid \! S_{i}\right]\mathbbm{1}\left(S_{i}\in\mathcal{S}_{\text{safe}}\right) \! \mid \! S_{i-1} \right]  \nonumber \\
&+ \mathbb{E}\left[G_{i}\nabla_{\theta}\log\pi_\theta(A_{i-1}\mid S_{i-1})\mid S_{i-1}\right].
\end{align}
This completes the proof of Lemma~\ref{lemma_safe_policy_gradient_G1}.
\end{proof}
\begin{lemma}
\label{lemma_nabla_E_G1_S0_GT_ST-1}
Given $S_{t-1} \in \mathcal{S}_\text{safe}$ and $G_{t}, t=1,2,\cdots, T-1$ defined in \eqref{def_G_cumulative_product}, it holds that
\begin{align}\label{eqn__nabla_E_G1_S0_GT_ST-1}
    \nabla_\theta\mathbb{E}\left[G_1\mid S_0\right] &=\sum\limits_{t=0}^{T-2}\mathbb{E}\left[G_1\nabla_{\theta}\log\pi_\theta(A_t\mid S_t)\mid S_0\right] \\ &+\mathbb{E}\left[\nabla_\theta\mathbb{E}\left[G_T \! \mid \! S_{T-1}\right]\! \prod_{t=1}^{T-1} \! \mathbbm{1}\left(\! S_{t}\in\mathcal{S}_{\text{safe}} \!\right) \! \mid \! S_0\right]  \nonumber. 
\end{align}
\end{lemma}
\begin{proof}
We proceed by employing Lemma~\ref{lemma_safe_policy_gradient_G1} to derive the gradient of the expectation of $G_1$ and $G_2$, respectively 
\begin{align}\label{eqn_mainbody_nabla_E_G1_S0}
\nabla_\theta\mathbb{E}\left[G_{1}\mid S_{0}\right] &= \mathbb{E}\left[\nabla_\theta\mathbb{E}\left[G_{2}\mid S_{1}\right]\mathbbm{1}\left(S_{1}\in\mathcal{S}_{\text{safe}}\right) \mid S_{0} \right] \nonumber \\ 
&+ \mathbb{E}\left[G_{1}\nabla_{\theta}\log\pi_\theta(A_{0}\mid S_{0})\mid S_{0}\right].
\end{align}
\begin{align}\label{eqn_mainbody_nabla_E_G2_S1}
   \nabla_\theta\mathbb{E}\left[G_2\mid S_1\right] &= \mathbb{E}\left[\nabla_\theta\mathbb{E}\left[G_3\mid S_2\right]\mathbbm{1}\left(S_2\in\mathcal{S}_{\text{safe}}\right) \mid S_1 \right] \nonumber \\
   &+ \mathbb{E}\left[G_2\nabla_{\theta}\log\pi_\theta(A_1\mid S_1)\mid S_1\right]. 
\end{align}
Then, substituting \eqref{eqn_mainbody_nabla_E_G2_S1} into \eqref{eqn_mainbody_nabla_E_G1_S0} yields
\begin{align}\label{eqn_mainbody_nabla_E_G1_S0_2}
    &\nabla_\theta\mathbb{E} [G_1\mid S_0 ] \nonumber \\
    &=\mathbb{E}[\mathbb{E} [\nabla_\theta\mathbb{E} [G_3\mid S_2 ]  \mathbbm{1} (S_2\in\mathcal{S}_{\text{safe}} ) \mid S_1  ]\mathbbm{1} (S_1\in\mathcal{S}_{\text{safe}} ) \nonumber \\
    &+\mathbb{E} [G_2\nabla_{\theta}\log\pi_\theta(A_1\mid S_1)\mid S_1 ]  \mathbbm{1} (S_1\in\mathcal{S}_{\text{safe}} ) \mid S_0] \nonumber \\
    &+\mathbb{E} [G_1\nabla_{\theta}\log\pi_\theta(A_0\mid S_0)\mid S_0 ].
\end{align}
As $\mathbbm{1}\left(S_1\in\mathcal{S}_{\text{safe}}\right)$ is measurable given $S_1$, we have
\begin{align}
    &\nabla_\theta\mathbb{E}\left[G_1\mid S_0\right] \nonumber \\
    &= \mathbb{E}[\mathbb{E}\left[\nabla_\theta\mathbb{E}\left[G_3\mid S_2\right]\mathbbm{1}\left(S_2\in\mathcal{S}_{\text{safe}}\right) \mathbbm{1}\left(S_1\in\mathcal{S}_{\text{safe}}\right) \mid S_1 \right] \nonumber \\
    & +\mathbb{E}\left[G_2\nabla_{\theta}\log\pi_\theta(A_1\mid S_1) \mathbbm{1}\left(S_1\in\mathcal{S}_{\text{safe}}\right) \mid S_1\right] \mid S_0] \nonumber \\ 
    & + \mathbb{E}\left[G_1\nabla_{\theta}\log\pi_\theta(A_0\mid S_0)\mid S_0\right].
\end{align}
By definition of $G_1$ we can simplify the second term of the right-hand side of the previous equation,
\begin{align}\label{eqn_appendix__nabla_E_G1_S0_3}
    &\nabla_\theta\mathbb{E}\left[G_1\mid S_0\right] \nonumber \\
    &= \mathbb{E}[\mathbb{E}\left[\nabla_\theta\mathbb{E}\left[G_3\mid S_2\right]\mathbbm{1}\left(S_2\in\mathcal{S}_{\text{safe}}\right) \mathbbm{1}\left(S_1\in\mathcal{S}_{\text{safe}}\right) \mid S_1 \right] \nonumber \\
    &+\mathbb{E}\left[G_1\nabla_{\theta}\log\pi_\theta(A_1\mid S_1) \mid S_1\right] \mid S_0] \nonumber \\
    &+ \mathbb{E}\left[G_1\nabla_{\theta}\log\pi_\theta(A_0\mid S_0)\mid S_0\right].
\end{align}
Using the towering property of expectation, \eqref{eqn_appendix__nabla_E_G1_S0_3} reduces to
\begin{align}
    &\nabla_\theta\mathbb{E}\left[G_1\mid S_0\right] \nonumber \\
    &= \mathbb{E}\left[\nabla_\theta\mathbb{E}\left[G_3\mid S_2\right]\mathbbm{1}\left(S_2\in\mathcal{S}_{\text{safe}}\right) \mathbbm{1}\left(S_1\in\mathcal{S}_{\text{safe}}\right) \mid S_0\right] \nonumber \\
    & +\mathbb{E}\left[G_1\nabla_{\theta}\log\pi_\theta(A_1\mid S_1) \mid S_0\right] \nonumber \\
    &+ \mathbb{E}\left[G_1\nabla_{\theta}\log\pi_\theta(A_0\mid S_0)\mid S_0\right].
\end{align}
Then repeatedly unwrapping $\nabla_\theta\mathbb{E}\left[G_1\mid S_0\right]$ in terms of $G_3,\ldots,G_T$ by Lemma~\ref{lemma_safe_policy_gradient_G1} yields
\begin{align}
    &\nabla_\theta\mathbb{E}\left[G_1\mid S_0\right] \nonumber \\ &=\mathbb{E} [\nabla_\theta\mathbb{E}\left[G_T\mid S_{T-1}\right] \mathbbm{1}\left(S_{T-1}\in\mathcal{S}_{\text{safe}}\right) \nonumber \\
    &\quad\cdots \mathbbm{1}\left(S_2\in\mathcal{S}_{\text{safe}}\right) \mathbbm{1}\left(S_1\in\mathcal{S}_{\text{safe}}\right) \mid S_0] \nonumber \\
    & +\mathbb{E}\left[G_1\nabla_{\theta}\log\pi_\theta(A_{T-2}\mid S_{T-2}) \mid S_0\right]+ \cdots \nonumber \\
    &+ \mathbb{E}\left[G_1\nabla_{\theta}\log\pi_\theta(A_0\mid S_0)\mid S_0\right] \nonumber \\
    &=\sum\limits_{t=0}^{T-2}\mathbb{E}\left[G_1\nabla_{\theta}\log\pi_\theta(A_t\mid S_t)\mid S_0\right] \nonumber \\
    &+ \mathbb{E}\left[\nabla_\theta\mathbb{E}\left[G_T\mid S_{T-1}\right]\prod_{t=1}^{T-1}\mathbbm{1}\left(S_{t}\in\mathcal{S}_{\text{safe}}\right) \mid S_0\right].
\end{align}
This completes the proof of Lemma~\ref{lemma_nabla_E_G1_S0_GT_ST-1}.
\end{proof}
We are now in conditions to prove Theorem~\ref{theorem_safe_policy_gradient}. We start by rewriting the probability of remaining safe in terms of $G_0$ defined in  \eqref{def_G_cumulative_product}. By definition of probability we have
\begin{align}
    &\mathbb{P} \left(\bigcap\limits_{t=0}^{T} \{ S_t \in \mathcal{S}_\text{safe}\} |\pi_\theta, S_0 \right) \nonumber \\
    &= \mathbb{E} \left[\mathbbm{1} \left(\bigcap\limits_{t=0}^{T} \{ S_t \in \mathcal{S}_\text{safe}  \}\right) |\pi_\theta, S_0 \right].
\end{align}
Note that the indicator function in the previous expression takes the value one, if and only if each $S_t\in\mathcal{S}_{\text{safe}}$. Hence, it is possible to rewrite the previous expression in terms of the product of indicator functions of states satisfying the safety condition at each time
\begin{align}\label{eqn_safe_policy_gradient_G0}
    \mathbb{P} \left(\bigcap\limits_{t=0}^{T} \{ S_t \in \mathcal{S}_\text{safe}\} |\pi_\theta, S_0 \right) &= \mathbb{E} \left[\! \prod\limits_{t=0}^{T} \! \mathbbm{1} (S_t \in \mathcal{S}_\text{safe}) |\pi_\theta, S_0 \right] \nonumber \\
    &= \mathbb{E}\left[G_0 | S_0\right],
\end{align}
where $\pi_\theta$ is omitted in the last equation for simplicity. By virtue of $S_0 \in \mathcal{S}_\text{safe}$, we obtain $ \mathbb{E}[G_0 | S_0] = \mathbb{E}[G_1 \cdot \mathbbm{1} (S_0 \in \mathcal{S}_\text{safe}) | S_0]=\mathbb{E}[G_1 | S_0].$ Then, using \eqref{eqn_safe_policy_gradient_G0}, the gradient of the probability of remaining safe reduces to
\begin{equation}\label{eqn_safe_policy_gradient_G1}
    \nabla_\theta \mathbb{P} \left(\bigcap\limits_{t=0}^{T} \{ S_t \in \mathcal{S}_\text{safe}\} |\pi_\theta, S_0 \right)=\nabla_\theta \mathbb{E}\left[G_1 | S_0\right].
\end{equation}
In Lemma~\ref{lemma_safe_policy_gradient_G1} we derive a recursive relationship for the gradient of $\mathbb{E}\left[G_t\mid S_{t-1}\right], t=1,2,\cdots, T-1 $. By virtue of Lemma~\ref{lemma_nabla_E_G1_S0_GT_ST-1}, to complete the proof of the result it suffices to establish that 
\begin{align}\label{eqn_thing_to_show}
    &\mathbb{E}\left[\nabla_\theta\mathbb{E}\left[G_T\mid S_{T-1}\right]\prod_{t=1}^{T-1}\mathbbm{1}\left(S_{t}\in\mathcal{S}_{\text{safe}}\right) \mid S_0\right] \nonumber \\
    &= \mathbb{E}\left[G_1\nabla_\theta \log \pi_{\theta}(A_{T-1} \mid S_{T-1})\mid S_0\right]. 
\end{align}
We establish this result next. Let us start by working with the gradient of the inner expectation on the left-hand side of the previous expression. 

Using the fact that $G_T = \mathbbm{1}\left(S_T\in\mathcal{S}_{\text{safe}}\right)$ and the definition of expectation one can write $\nabla_\theta\mathbb{E}\left[G_T\mid S_{T-1}\right]$ in the left-hand side of the previous expression as
\begin{equation}\label{eqn_nabla_G_T}
   \nabla_\theta\mathbb{E}\left[G_T\mid S_{T-1}\right] = \nabla_\theta \int_{\mathcal{S}} \mathbbm{1} (s_T \in \mathcal{S}_\text{safe}) p(s_T | S_{T-1}) \, ds_T, 
\end{equation}
where $p\left(s_T\mid S_{T-1}\right)$ denotes the conditional probability of $S_T$ given $S_{T-1}$. Marginalizing the probability distribution it follows that 
\begin{align}
    p(s_T | S_{T-1}) \! = \! \int_{\mathcal{A}} \! p(s_T | S_{T-1}, a_{T-1})  \pi_{\theta}(a_{T-1} | S_{T-1}) da_{T-1}.
\end{align}
Consequently, \eqref{eqn_nabla_G_T} can be converted to
\begin{align}
   &\nabla_\theta\mathbb{E}\left[G_T\mid S_{T-1}\right] \nonumber \\
   &= \nabla_\theta\int_{\mathcal{S}\times \mathcal{A}} \mathbbm{1}\left(s_T \in \mathcal{S}_{\text{safe}}\right) p(s_T \mid S_{T-1}, \, a_{T-1}) \nonumber \\
   &\quad \quad \quad \quad \quad \quad \pi_{\theta}(a_{T-1}\mid S_{T-1}) \, ds_T da_{T-1}.
\end{align}
Note that in the previous expression, the only term dependent on $\theta$ is the policy, hence we have that
\begin{align}\label{eqn_gradient_GT}
   \nabla_\theta\mathbb{E}\left[G_T\mid S_{T-1}\right]  =\int_{\mathcal{S}\times \mathcal{A}}\! &\mathbbm{1}\left(s_T \in \mathcal{S}_{\text{safe}}\right) p(s_T \! \mid \! S_{T-1}, \, a_{T-1}) \nonumber \\
   &\nabla_\theta\pi_{\theta}(a_{T-1}\mid S_{T-1}) \, ds_T da_{T-1}. 
\end{align}
Applying the ``log-trick'' to the right-hand side of \eqref{eqn_gradient_GT} yields
\begin{align}
   &\nabla_\theta\mathbb{E}\left[G_T\mid S_{T-1}\right] \nonumber \\ &=\int_{\mathcal{S}\times \mathcal{A}}\mathbbm{1}\left(s_T \in \mathcal{S}_{\text{safe}}\right) p(s_T \mid S_{T-1}, \, a_{T-1}) \nonumber \\
   &\quad~\pi_{\theta}(a_{T-1}\mid S_{T-1}) \nabla_\theta\log\pi_{\theta}(a_{T-1}\mid S_{T-1}) \, ds_T da_{T-1}. 
   \end{align}
Since $p\!\left(s_T| S_{T-1},\!a_{T-1}\right)\! \pi_\theta \!\left(a_{T-1}| S_{T-1}\right)\! =\! p\left(s_T,a_{T-1}| S_{T-1}\right)$ is the joint probability distribution of $S_{T}$ and $A_{T-1}$ given $S_{T-1}$ the previous expression reduces to
\begin{equation}
    \nabla_\theta\mathbb{E}\left[G_T\mid S_{T-1}\right]  =\mathbb{E}\left[G_T\nabla_{\theta}\log\pi_\theta(A_{T-1}\mid S_{T-1})\mid S_{T-1}\right]. 
\end{equation}

Since $S_1,\ldots, S_{T-1}$ are measurable with respect to $S_{T-1}$ it follows that 
\begin{align}
   &\nabla_\theta\mathbb{E}\left[G_T\mid S_{T-1}\right] \prod_{t=1}^{T-1}\mathbbm{1}\left(S_{t}\in\mathcal{S}_{\text{safe}}\right) \nonumber \\
   &=\mathbb{E}\left[G_1\nabla_{\theta}\log\pi_\theta(A_{T-1}\mid S_{T-1})\mid S_{T-1}\right],
\end{align}
where we have used that $G_1 = G_T\prod_{t=1}^{T-1}\mathbbm{1}\left(S_{t}\in\mathcal{S}_{\text{safe}}\right)$. Substituting the previous expression in the left-hand side of \eqref{eqn_thing_to_show} it follows that 
\begin{align}
    &\mathbb{E}\left[\nabla_\theta\mathbb{E}\left[G_T\mid S_{T-1}\right]\prod_{t=1}^{T-1}\mathbbm{1}\left(S_{t}\in\mathcal{S}_{\text{safe}}\right) \mid S_0\right] \nonumber \\
    &=\mathbb{E}\left[\mathbb{E}\left[G_1\nabla_\theta \log \pi_{\theta}(A_{T-1} \mid S_{T-1}) \mid S_{T-1}\right]\mid S_0\right]. 
\end{align}
The law of total expectation completes the result claimed in \eqref{eqn_thing_to_show} and therefore completes the proof of Theorem~\ref{theorem_safe_policy_gradient}.

\end{proof}
The expression for the gradient in Theorem~\ref{theorem_safe_policy_gradient}, allows us to directly tackle safe RL problems that take the form of \eqref{eqn_problem1} using policy optimization techniques. It is worth pointing out that the proof of Theorem~\ref{theorem_safe_policy_gradient} is similar to policy gradient theorems in the literature~\cite{williams1992simple,sutton1999policy}.
Although promising, stochastic approximations of the gradient introduce challenges. Unlike the classic policy gradient (where policy parameters update each iteration), under this framework, the parameters in \eqref{eqn_theorem} only update when every step of the trajectory is safe, i.e., when $G_1=\prod_{t=1}^T\mathbbm{1}\left(S_t\in\mathcal{S}_{\text{safe}}\right)=1$. Addressing this issue is beyond
the scope of this work and opens up several interesting future research avenues. Despite this limitation, we demonstrate in the next section that the estimate proposed can solve problems of the form \eqref{eqn_problem1}.

\section{Numerical Experiments}
\label{Numerical_Experiments}
To illustrate the ability of using \eqref{eqn_theorem} to
train safe policies, we consider a continuous navigation task in an environment populated with hazardous obstacles (see Figure~\ref{configuration}). The coordinates of the obstacles' centers are (7, 7), (3, 7), (1.5, 4), (4.5, 3), (8, 3) with the corresponding radii \{2, 1, 0.5, 1.5, 0.75\}. The state in this example  is  the position of the agent on the $x-$ and $y-$axis, namely, $s=(x, y)$. We set the continuous state space as $\mathcal{S}=[0, 10] \times [0, 10]$. The goal of the agent is to reach a goal position $s_{goal}=(9, 1.5)$ within the time horizon $T=20$, while avoiding 5 obstacles. Accordingly, the safe set is defined as the whole map/state space excluding regions of 5 obstacles.

The agent's action $a$ is a two-dimensional velocity. For a given state and action at time $t$, the state evolves according to $s_{t+1}=s_t+ a_t T_s$ with $T_s=0.05$. The policy of the agent is a multivariate Gaussian distribution
\begin{align}\label{gaussian_policy}
    \pi_\theta (a | s)\! =\! \frac{1}{\sqrt{2\pi |\Sigma|}} \exp \! \left(\! -\frac{1}{2} \left(a-\mu_\theta(s)\right)^\top \!\! \Sigma^{-1} \!\! \left(a-\mu_\theta(s)\right) \! \right),
\end{align}
where $\mu_\theta(s)$ and $\Sigma$ denote the mean and covariance matrix of the Gaussian policy. We parameterize $\mu_{\theta}(s)$ as a linear combination of Radial Basis Functions (RBFs) 
\begin{align}\label{RBF}
    \mu_\theta (s) = \sum_{k=1}^{d} \theta_k \exp \left(-\frac{||s-\Bar{s}_k||^2}{2\sigma^2} \right),
\end{align}
where $\theta = [\theta_1, \theta_2, \cdots, \theta_d]^\top$ are parameters that need to be learned, $\Bar{s}_k$ are centers of each RBFs kernel and $\sigma$ their bandwidth. In this experiment we set $\Sigma = \text{diag} (0.5, 0.5)$, $\sigma =0.5$, $d=1681$  and $\Bar{s}_k = (x_k, y_k),  k=1, 2, \cdots, 1681$ where $\Bar{s}_k$ forms a uniform lattice with separation $0.25$ in each direction. The reward is the negative squared distance to the goal position $s_{goal}$, i.e., $r(s_t, a_t)= -\lVert s_t - s_{goal} \rVert^2$. 
\begin{figure}
\centering
\includegraphics[width=\columnwidth]{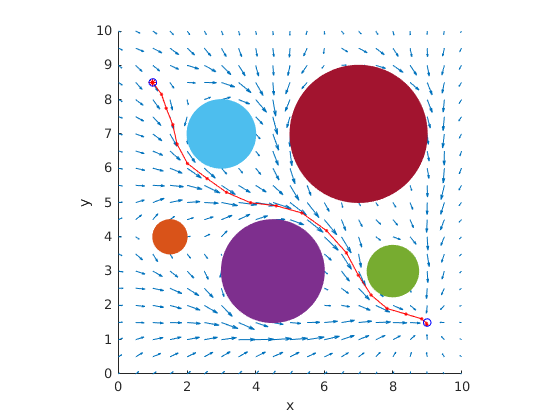}
\caption{Navigation policy learned after 40,000 iterations for probabilistic constraint formulation selecting $\lambda=6, \eta = 0.002$. The agent is trained to navigate starting from (1, 8.5) to a goal (9, 1.5).}
\label{configuration}
\end{figure} 
\begin{figure}[htbp]
\centering 
\subfigure[Evolution of the averaged cumulative reward] 
{
\begin{minipage}{9cm}
\centering    
\includegraphics[scale=0.55]{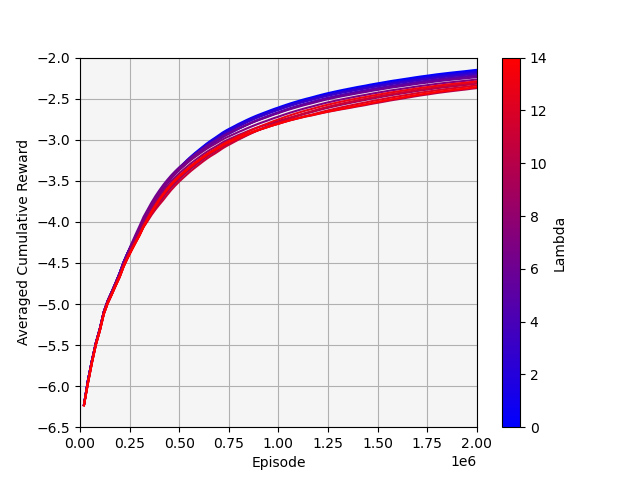}  
\end{minipage}
}
\subfigure[Evolution of the averaged safety probability] 
{
	\begin{minipage}{9cm}
	\centering 
	\includegraphics[scale=0.55]{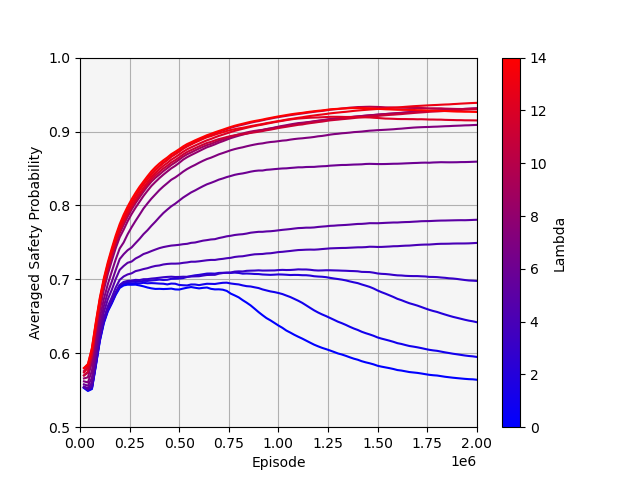}  
	\end{minipage}
}
\caption{{Evolution of the averaged cumulative reward and averaged safety probability as the algorithm iterates. The step-size $\eta$ and time horizon $T$ are fixed to 0.0006 and 2000000. The color bar represents that $\lambda$ is selected from $[0.5, 14]$ in which red and blue denote large and small values of  $\lambda$ respectively.} }
\label{training_curve}
\end{figure} 
%
%
%
%
%
%
%
%
%
%
%
%

To solve problem~\eqref{eqn_agmt_obj} in this set-up, we consider a stochastic approximation of the gradient ascent, which yields the following update rule for the parameters $\theta$ of the policy
%
\begin{align}\label{update_rule_path}
    \theta_{k+1} = \theta_k + \eta \left( \hat{\nabla}_\theta V(\theta_k) +  \lambda \hat{\nabla}_\theta \mathbb{P} \left(\bigcap\limits_{t=0}^{T} \{ S_t \in \mathcal{S}_\text{safe}\}\right)   \right),
\end{align}
where the first term in bracket on the right-hand side is computed by a stochastic approximation of the Policy Gradient Theorem~\cite{sutton1999policy} and the second term is a stochastic approximation of \eqref{eqn_theorem}.

Figure~\ref{configuration} demonstrates that the agent with probabilistic safety constraints is trained to safely navigate to the goal position (9, 1.5) from the initial state (1, 8.5) after 40,000 episodes of training, during which $\lambda$ is fixed to be 6 and with $\eta = 0.002$.

Note that to attain different levels of safety, in general, different values of $\lambda$ are required. Subsequently, we run \eqref{update_rule_path} with different $\lambda$ to find solutions to problem \eqref{eqn_agmt_obj}. The worst-case scenario requires $\eta = 0.0006$ and 2,000,000 episodes. As depicted in Figure~\ref{training_curve}, the color bar represents that $\lambda$ is chosen from [0.5, 14] in which red and blue denote large and small values of $\lambda$ respectively. To quantitatively analyze the effect of $\lambda$ on safety and task performance, the safety probability is defined as the fraction of safe episodes over 1000 independent episodes under different random seeds and then averaged across the evaluation episodes. Similarly, we estimate the value function by averaging the cumulative reward across the evaluation episodes. It is not surprising that larger $\lambda$ yields larger safety probability and lower cumulative rewards.

We thus demonstrated that stochastic approximations of the gradient of the probabilistic constraints (Theorem~\ref{theorem_safe_policy_gradient}) can successfully be employed for solving the task at hand. Notice that the algorithms used in this numerical section are Monte Carlo methods~\cite[Chapter 5]{sutton2018reinforcement}. In the RL literature there exist algorithms that exploit temporal differences~\cite{watkins1992q} and/or trust regions~\cite{schulman2015trust} which result in faster convergence rates.  As mentioned in Section \ref{Gradient of the Probabilistic Constraint} the estimation of the gradient of the probabilistic constraint is zero for every unsafe episode, thus hindering the rate of convergence. Analyzing alternatives to overcome this issue is out of the scope of this work.


\section{Conclusions}\label{sec_conclusions}

In this work, we considered the problem of learning probabilistic safe policies. Unlike cumulative constraints often considered in the literature, we aim to guarantee that the state of the agent remains in the safe set with high probability.

We have provided the first expression for the gradient of a probabilistic constraint safety requirement, thus enabling the application of Policy Optimization methods in these settings as well. We have also demonstrated that updates based on this gradient can be used to solve continuous navigation problems in cluttered environments. The stochastic approximation presented is not without issues. In particular, the gradient estimate is zero unless the agent remains on the safe set during the episode. Future work includes improving this estimate and characterizing the convergence and data-efficiency of algorithms that use this gradient.

\bibliographystyle{ieeetr}
\bibliography{bib}

\vspace{12pt}

\end{document}